\documentclass{new_tlp}

\usepackage{amssymb}
\usepackage{amsmath}
\usepackage{url}
\usepackage{xspace}
\usepackage{paralist}
\usepackage{mathptmx}

\newcommand{\pr}[1]{\m{\ensuremath{\mathsf{#1}}}}
\newcommand{\m}[1]{\ensuremath{#1}\xspace}
\newcommand{\limplies}{\m{\Rightarrow}}
\newcommand{\lequiv}{\m{\Leftrightarrow}}
\newcommand{\lrule}{\m{\leftarrow}}
\newcommand{\struct}{\m{I}}
\newcommand{\PP}{\m{\mathcal{P}}}
\newcommand{\system}[1]{\textsc{#1}\xspace}
\newcommand{\orderof}[1]{\mathcal{O}(#1)}
\newcommand{\theory}{\m{T}}
\newcommand\EAQBF{\m{\exists\forall\text{QBF}}}
\newcommand{\voc}{\m{\sigma}}
\newcommand{\arc}[3]{{#1}\stackrel{#2}{\longrightarrow}{#3}}
\newcommand{\head}{\m{\mathit{head}}}
\newcommand{\body}{\m{\mathit{body}}}
\newcommand{\infpred}{\pr{inf}}
\newcommand{\CC}{\m{\mathcal{C}}}

\newcommand\citet[1]{\citeNS{#1}\xspace} % Should be fixed otherwise

\newtheorem{theorem}{Theorem}[section]
\newtheorem{definition}[theorem]{Definition}
\newtheorem{example}[theorem]{Example}
\newtheorem{proposition}[theorem]{Proposition}

\title[Stable-Unstable Semantics: 
Beyond NP with Normal Logic Programs]{%
Stable-Unstable Semantics: \\
Beyond NP with Normal Logic Programs}

\author[B. Bogaerts and T. Janhunen and S. Tasharrofi]{%
BART BOGAERTS, TOMI JANHUNEN and SHAHAB TASHARROFI\\
Helsinki Institute for Information Technology HIIT%
\thanks{The support from the Finnish Center of Excellence in Computational
Inference Research (COIN) funded by the Academy of Finland
(under grant \#251170) is gratefully acknowledged.}
Department of Computer Science \\
Aalto University, FI-00076 AALTO, Finland\\
\email{firstname.lastname@aalto.fi}}

\jdate{March 2003}
\pubyear{2003}
\pagerange{\pageref{firstpage}--\pageref{lastpage}}
\doi{S1471068401001193}

\begin{document}

\label{firstpage}

\maketitle

\begin{abstract}
Standard answer set programming (ASP) targets at solving search
problems from the first level of the polynomial time hierarchy (PH).
Tackling search problems beyond NP using ASP
is less straightforward.
The class of disjunctive logic programs offers the most prominent way
of reaching the second level of the PH, but encoding respective hard
problems as disjunctive programs typically requires sophisticated
techniques such as saturation or meta-interpretation. The
application of such techniques easily leads to encodings that are
inaccessible to non-experts. Furthermore, while disjunctive ASP
solvers often rely on calls to a (co-)NP oracle, it may be difficult
to detect from the input program where the oracle is being accessed.
In other formalisms, such as Quantified Boolean Formulas (QBFs), the
interface to the underlying oracle is more transparent as it is
explicitly recorded in the quantifier prefix of a formula. On the
other hand, ASP has advantages over QBFs from the modeling
perspective. The rich high-level languages such as ASP-Core-2
offer a wide variety of primitives that enable concise and natural
encodings of search problems.
In this paper, we present a novel logic programming--based modeling
%BART: for the weird double dash above, see http://english.stackexchange.com/questions/29036/hardware-counter-based-tools-or-hardware-counter-based-tools
paradigm that combines the best features of ASP and QBFs. We develop
so-called \emph{combined logic programs} in which oracles are directly
cast as (normal) logic programs themselves. Recursive incarnations of
this construction enable logic programming on arbitrarily high levels
of the PH. We develop a proof-of-concept implementation for our new paradigm. 
This paper is under consideration for acceptance in TPLP.
\end{abstract}

\begin{keywords}
disjunctive logic programming,
polynomial hierarchy, 
quantified Boolean formulas%,
% QBF    
\end{keywords}

%-----------------------------------------------------------------------------
  
\section{Introduction}
\label{sec:intro}

With the launch of the idea that stable models \cite{GL88:iclp}
of a logic program can be used to encode search problems, a new
programming paradigm, called Answer Set Programming (ASP) was born
\cite{MT99,Niemela99:amai,Lifschitz99:iclp}.
Nowadays, the fact that normal logic programs can effectively encode
NP-complete decision and function problems is exploited in
applications in many different domains such as
robotics \cite{ARSS15:lpnmr},
machine learning \cite{JGRNPC15:sc,BBBDDJLRDV15:tplp},
phylogenetic inference \cite{KOJS15:tplp,BEEMR07:jar},
product configuration \cite{TSNS03:iced},
decision support for the Space Shuttle \cite{NBGWB01:padl},
e-Tourism \cite{RDGIIML10:fi},
and knowledge management \cite{GILR09:lpnmr}.

Tackling search problems beyond NP with ASP requires one to use more
expressive logic programs than the normal ones. To this end, the class
of disjunctive programs \cite{GL91:ngc} is the most prominent
candidate. As shown by \citet{EG95:amai}, the main decision
problems associated to disjunctive programs are $\Sigma^P_2$- and
$\Pi^P_2$-complete, depending on the reasoning mode, i.e.,
\emph{credulous} vs.\ \emph{cautious} reasoning.
But when it comes to applications, one encounters disjunctive
encodings less frequently than encodings as normal logic
programs. This is also witnessed by the benchmark problems submitted
to ASP competitions \cite{CGMR16:aij}. Such a state of affairs
is not due to a lack of application problems since many complete
problems from the second level of the PH are known. Neither is it due
to a lack of implementations, since state-of-the-art ASP solvers such
as
\system{dlv} \cite{LPFEGPS06:acmtocl} and \system{clasp}
\cite{DGGKKOS08:kr,GKKRS15:lpnmr} offer a seamless support for
disjunctive programs.

An explanation for the imbalance identified above can be found in the
essentials of disjunctive logic programming when formalizing problems
from the second level of the PH.  There are results \cite{BED94:amai}
showing that such programs must involve \emph{head cycles}, i.e.,
cyclic positive dependencies established by the rules of the program
that intertwine with the disjunctions in the program. Such
dependencies may render disjunctive programs hard to understand and to
maintain.
Moreover, the existing generic encodings of complete problems from the
second level of the PH as disjunctive programs are based on sophisticated
\emph{saturation} \cite{EG95:amai} or
\emph{meta-interpretation} \cite{GKS11:tplp}
techniques, which may turn an encoding inaccessible
to a non-expert. \citet{EP06:tplp} identify the
limitations of subprograms that act as (co-)NP-oracles and are embedded
in disjunctive programs using the saturation technique.
Summarizing our observations, the access to the underlying oracle is
somewhat cumbersome and difficult to detect from a given disjunctive
program. Interestingly, the oracle is better visible in native
implementations of disjunctive logic programs
\cite{JNSSY06:tocl,DGGKKOS08:kr}
where two ASP solvers cooperate: one is responsible for
\emph{generating} model candidates and the other for \emph{testing}
the minimality of candidates. In such an architecture, a successful
minimality test amounts to showing that a certain subprogram has no
stable models.

In other formalisms, the second level of the PH is reached differently.
For instance, \emph{quantified Boolean formulas} (QBFs)
\cite{SM73:stoc},
record the interface between existentially and universally quantified
subtheories, intuitively corresponding to the generating and testing
programs mentioned above, explicitly in the quantifier prefix of the
theory. From a modelling perspective, on one hand, QBFs support the
natural formalization of subproblems as subtheories and the
quantifications introduced for variables essentially identify the
oracles involved.
On the other hand, logic programs also have some advantages over QBFs.
Most prominently, they allow for the natural encodings of
\emph{inductive definitions}, not to forget about \emph{default
  negation}, \emph{aggregates} and \emph{first-order features}
available in logic programming.  The rich high-level modelling
languages such as ASP-Core-2 \cite{aspcore2} offer a wide variety of
primitives that are not available for QBFs and require substantial
elaboration if expressed as part of a QBF.

In this paper, we present a novel logic programming--based modeling
paradigm that combines the best features of ASP and QBF. We introduce
the notion of a \emph{combined logic program} which explicitly
integrates a normal logic program as an oracle to another program. The
semantics of combined programs is formalized as
\emph{stable-unstable models}
whose roots can be recognized from earlier work of
\citet{EP06:tplp}.
Our design directly reflects the generate-and-test methodology
discussed above, enabling one to encode problems up to the second
level of the PH. Compared to disjunctive programs, our approach is
thus closer to QBFs and if the same design is applied recursively, our
new formalism can be adapted to tackle problems arbitrarily high in
the PH, in analogy to QBFs. We develop a proof-of-concept solver for
our new formalism on top of the recently introduced
solver \system{sat-to-sat} \cite{JTT16:aaai},
which is based on an architecture of two interacting, conflict-driven
clause learning (CDCL) SAT solvers. The solver capable of searching
for stable-unstable models is obtained using the methodology of
\citet{BJT16:kr}, who automatically translate a second-order
specification, combined with data that represents the involved ground
programs in a reified form, into a \system{sat-to-sat} specification. The details of
the solver architecture are hidden from the user so that a user
experience similar to native ASP solvers is obtained, where the user
inputs two logic programs in a familiar syntax and the solver produces
answer sets.

The rest of this paper is structured as follows.
In Section \ref{sec:related}, we discuss related work in more detail.
We recall some basic notions of logic programs in Section
\ref{sec:lp}.
Afterwards, in Section \ref{sec:new}, we present our new logic
programming methodology. We illustrate how it can be used to tackle
some problems from the second level of the PH in Section
\ref{sec:apps}.
In Section \ref{sec:impl}, we
show how our new formalism can
be implemented on top of \system{sat-to-sat}.
We show how our formalism naturally extends beyond the second level
of the PH in Section \ref{sec:beyond}
and conclude the paper in Section \ref{sec:concl}.

%-----------------------------------------------------------------------------

\section{Related Work}
\label{sec:related}

A fundamental technique to encode $\Sigma^P_2$-complete problems as
disjunctive programs is known as \emph{saturation}.  The technique
goes back to the $\Sigma^P_2$-completeness proof for the existence of
stable models in the case of disjunctive programs
\cite{EG95:amai}.
Although saturation can be applied in a very systematic fashion to some
programs of interest,
\citet{EP06:tplp}
identify the impossibility of having negation as a central limitation
of oracles encoded by saturation, rendering the oracle call to a bare
minimality check rather than showing that an oracle program has no
stable models. This limitation can be partially circumvented using
\emph{meta-interpretation}
\cite{EP06:tplp,GKS11:tplp},
but these techniques do not necessarily decrease the
\emph{conceptual complexity}
of disjunctive programming from the user's perspective.

The approach of \citet{EP06:tplp} is perhaps most closely
related to our work. They present a transformation of two
\emph{head-cycle free} (HCF) disjunctive logic programs
$(\PP_g,\PP_t)$, where $\PP_g$ and $\PP_t$ form the \emph{generating}
and \emph{testing} programs, into a disjunctive program $\PP_{c}$.
In our terminology, the \emph{stable-unstable} models of the
\emph{combined program} $(\PP_g,\PP_t)$ are in one-to-one
correspondence with the stable models of $\PP_c$. Thus, their approach
is based on essentially the same base definition.
However, their transformation counts on meta-interpretation and
$\PP_c$ is encoded as a disjunctive meta program to capture the
intended semantics of $(\PP_g, \PP_t)$. A similar meta-encoding can be
obtained using the approach of \citet{GKS11:tplp}, but stable-unstable
semantics is not explicit in their work. Since these meta programming
approaches use disjunctive logic programs as the back end formalism,
they are inherently confined to the second level of the PH.
Our approach, on the other hand, easily generalizes for the classes of
the entire PH, as to be shown in Section~\ref{sec:beyond}. Moreover,
when
\citet{EP06:tplp}
translate $(\PP_g,\PP_t)$ into a disjunctive logic program the
essential structural distinction between $\PP_g$ and $\PP_t$ is lost.
Many disjunctive answer set solvers
\cite{JNSSY06:tocl,DGGKKOS08:kr}
try to recover this interface due to their internal data
structures. In our approach, the generate-and-test structure of the
original problem is explicitly present in the input presented to the solver. 

While meta programming can be viewed as a front-end to disjunctive
logic programming, the goal of our work is to foster the idea of
generate-and-test programs as a basis for a logic programming
methodology that complexity-wise covers the entire PH. In this paper,
we present a proof-of-concept implementation based on the recursive
\system{sat-to-sat} solver architecture \cite{JTT16:aaai,BJT16:bnp}. It is
reasonable to expect that such an architecture can be realized in the
future using native ASP solvers as building blocks, too, thus
eliminating the need for second-order interpretation.

Another formalization of a similar idea was worked out by
\citet{EGV97:lpnmr}, based on the theory of generalized quantifiers
\cite{Mostowski57:fm,Lindstrom66:th}. The semantics we propose for combined
logic programs can be obtained as a special case of a (stratified)
logic program with generalized quantifiers \cite{EGV97:lpnmr}. One
important difference is that in our approach, the interaction between
the two programs is fixed: one program serves as generator and the
second as a tester program.  The approach of \citet{EGV97:lpnmr} is more
general in the sense that it allows for other types of interaction as
well.  The price to pay for this generality is that the interaction
between programs needs to be specified explicitly by users, resulting
in a more error-prone modelling process. Moreover, in our approach,
the input expected from the user is a set of source files in a
familiar syntax (ASP-Core-2), requiring no syntactic extension for
quantification.

%-----------------------------------------------------------------------------

\section{Preliminaries: Logic Programming}
\label{sec:lp}

In this section, we recall some preliminaries from logic programming.
The new semantics is only formulated for propositional programs but,
in practice, the users are not expected to write propositional
programs. Instead, they are supposed to use grounders, such as the
state-of-the-art grounder \system{Gringo}, to transform first-order programs
to propositional ones.

A \emph{vocabulary} is a set of symbols, also called \emph{atoms};
vocabularies are denoted by $\voc,\tau$. A \emph{literal} is an atom
or its negation.  A \emph{logic program} $\PP$ over vocabulary $\voc$
is a set of \emph{rules} $r$ of form
\begin{equation}\label{eq:rule} 
h_1\lor \dots \lor h_l
\lrule a_1\land \dots \land a_n \land \lnot b_1\land \dots \land \lnot b_m.
\end{equation}
where $h_i$'s, $a_i$'s, and $b_i$'s are atoms in $\voc$. We call
$h_1\lor \dots \lor h_l$ the \emph{head} of $r$, denoted $\head(r)$,
and $a_1\land \dots \land a_n \land \lnot b_1\land \dots \land \lnot
b_m$ the \emph{body} of $r$, denoted $\body(r)$. A program is
\emph{normal} (resp. \emph{positive}) if $l=1$ (resp. $m=0$) for all
rules in \PP. If $n=m=0$, we simply write $h_1\lor \dots \lor h_l$.

An interpretation $\struct$ of a vocabulary \voc is a subset of
$\voc$.  An interpretation $I$ is a \emph{model} of a logic program
\PP if, for all rules $r$ in \PP, whenever $\body(r)$ is satisfied by
$I$, so is $\head(r)$.  The \emph{reduct} of \PP with respect to $I$,
denoted $\PP^I$, is the program that consists of rules $ h_1\lor \dots
\lor h_l \lrule a_1\land \dots \land a_n $ for all rules of the form
\eqref{eq:rule} in \PP such that $b_i\not\in I$ for all $i$.
An interpretation $I$ is a \emph{stable model} of \PP if it is a
$\subseteq$-minimal model of $\PP^I$ \cite{GL88:iclp}.

\emph{Parameterized logic programs} have been implicitly present in
the literature for a long time, by assigning a meaning to
\emph{intensional} databases. They have been made explicit in various
forms
\cite{GP96:ijseke,OJ06:ecai,DV07:lpnmr,DLTV12:iclp}.
We briefly recall the basics. Assume that $\tau\subseteq \voc$ and
$\PP$ is a logic program over $\voc$ such that no atoms from $\tau$
occur in the head of a rule in $\PP$.  We call $I$ a
\emph{parameterized stable model} of $\PP$ with respect to
\emph{parameters} $\tau$ if $I$ is a stable model of
$\PP\cup (I\cap \tau)$.
Parameters $\tau$ are also known as \emph{external}, \emph{open}, or
\emph{input atoms}.

From time to time, we use syntactic extensions such as choice rules,
constraints, and cardinality atoms in this paper.
A \emph{cardinality atom} $m\leq \#\{l_1, \dots, l_n\} \leq k$ (with
$l_1, \dots, l_n$ being literals and $m, k \in \mathbb{N}$) is
satisfied by $I$ if $m \leq \#\{i\mid l_i\in I\} \leq k$. A
\emph{choice rule} is a rule with a cardinality atom in the head. A
\emph{constraint} is a rule with an empty head. An interpretation $I$
satisfies a constraint $c$ if it does not satisfy $body(c)$.  These
language constructs can all be translated to normal rules
\cite{BJ13:lpnmr}.  We also sometimes use the colon syntax $H :
L$ for conditional literals as a way to succinctly specify a set of
literals in the body of a rule or in a cardinality atom
\cite{GHKLS15:tplp}.

%-----------------------------------------------------------------------------

\section{Stable-Unstable Semantics}
\label{sec:new}

The design goal of our new formalism is to isolate the logic program
that is acting as an oracle for another program. Thus, we would like
to find a stable model $I$ for a program while showing the
\emph{non-existence} of stable models for the oracle program given $I$.
Following this intuition, we formalize the pair $(\PP_g,\PP_t)$ of a
\emph{generating} program $\PP_g$ and a \emph{testing} program $\PP_t$
as follows.%
\footnote{The terminology goes back to \system{GnT}, one of the early
solvers developed for disjunctive programs \cite{JNSSY06:tocl}.}

\begin{definition}[Combined logic program]
\label{def:combined-program}
A \emph{combined logic program} is pair $(\PP_g,\PP_t)$ of normal logic
programs $\PP_g$ and $\PP_t$ with vocabularies $\voc_g$ and
$\voc_t$ such that $\PP_g$ is parameterized by $\tau_g\subseteq\voc_g$
and $\PP_t$ is parameterized by $\voc_g\cap \voc_t$.
\end{definition}

The vocabulary of the program $(\PP_g,\PP_t)$ is $\voc_g$; it consists
of all symbols that are ``visible'' to the outside. Symbols in
$\voc_t\setminus \voc_g$ are considered to be \emph{quantified
  internally}.  The use of normal programs in the definition of
combined logic programs, or \emph{combined programs} for short, is a
design decision aiming at programs that are easily understandable
(compared to, for instance, disjunctive programs with head-cycles).
In principle, our theory also works when replacing normal programs
with another class of programs.
Our next objective is to define the semantics of combined programs
which should not be a surprise given the above intuitions.

\begin{definition}[Stable-unstable model]\label{def:semantics}
Given a combined program $(\PP_g,\PP_t)$ with vocabularies
$\voc_g$ and $\voc_t$, a $\voc_g$-interpretation $I$ is
a \emph{stable-unstable model} of $(\PP_g,\PP_t)$ if
the following two conditions hold:
\begin{enumerate}
\item
$I$ is a parameterized stable model of $\PP_g$ with respect to
$\tau_g$ (the parameters of $\PP_g$) and

\item
there is no parameterized stable model $J$ of $\PP_t$ that coincides
with $I$ on $\voc_t\cap\voc_g$ (i.e., such that $I\cap
{\voc_t}=J\cap {\voc_g}$).
\end{enumerate}
\end{definition}

The fact that a $\voc_g$-interpretation $I$ is a stable-unstable
model of $(\PP_g,\PP_t)$ is denoted $I\models_{su}
(\PP_g,\PP_t)$. Note that the testing program stands for the
\emph{non-existence} of stable models.
If $\voc_g\cap\voc_t\neq\emptyset$, the programs truly
interact. Otherwise, we call $(\PP_g,\PP_t)$ \emph{independent}.

\begin{example}\label{ex:small}
Let 
$\PP_1=\{0\leq \#\{c\}\leq 1.~~
         \lrule c \land d.~~
         \lrule \lnot c\land b.\}$
and
$\PP_2=\{0\leq \#\{a\} \leq 1.~~
         b \lrule a.\}$
where $\PP_1$ has vocabulary $\voc_1=\{c,b,d\}$ and
parameters $\tau_1=\{b,d\}$, and $\PP_2$ has vocabulary
$\voc_2=\{a,b,d\}$ and parameters $\tau_2=\{d\}$. The stable
models of $\PP_1$ and $\PP_2$ are, respectively,
$\{\{d\},\{b,c\},\{\},\{c\}\}$ and $\{\{d,a,b\}, \{d\}, \{a,b\},
   \{\}\}$.
Notice that $\tau_1=\voc_1\cap\voc_2$. The combined program
$(\PP_2,\PP_1)$ has parameters $\tau_2$ and has only one
stable-unstable model $\{d,a,b\}$ since all other stable models of
$\PP_2$ coincide with a stable model of $\PP_1$ on $\tau_1$.
\end{example}

\begin{theorem}\label{thm:complex}
Deciding the existence of a stable-unstable model for a \emph{finite}
combined program $(\PP_g,\PP_t)$ is $\Sigma^P_2$-\allowbreak{}complete
in general, and $D^P$-\allowbreak{}complete for
\emph{independent} combined programs.
\end{theorem}
\begin{proof}
The theorem is a straightforward consequence of known complexity results.
The membership in $\Sigma^P_2$ follows directly from the definition of
$\Sigma^P_2$ and the fact that deciding whether a normal logic program
has a stable model is NP-complete \cite{MT99}.
For hardness in the general case, we recall that
\citet{JNSSY06:tocl}
have shown that any disjunctive logic program $\PP$ can be represented
as a pair of normal programs $(\PP_g,\PP_t)$ whose stable-unstable
models essentially capture the stable models of $\PP$.

In the case of an independent input $(\PP_g,\PP_t)$, the decision
problem conjoins an NP-complete problem (showing that $\PP_g$ has a
stable model) and a co-NP-complete problem (showing that $\PP_t$ has
no stable models). Thus, membership in $D^p$ is immediate.
The hardness is implied by Niemel\"a's reduction \citeyear{Niemela99:amai}
that translates a set of clauses $C$ into a normal logic program
$N(C)$, when applied to instances of the $D^P$-\allowbreak{}complete
SAT-UNSAT problem.
\end{proof}

\begin{example}\label{ex:eaqbf}
Any \EAQBF of the form $\exists\vec{x}\forall\vec{y}: \varphi$ with
$\varphi$ a Boolean formula in DNF can be encoded as a combined
program as follows.
Let $\PP_g$ be a logic program that expresses the choice of a truth
value for every variable $x$ in $\vec{x}$ using two normal rules
$x\lrule \lnot x'$ and $x'\lrule \lnot x$ where $x'$ is new.
Also, let $\PP_t$ be a logic program that similarly chooses truth
values for every $y$ in $\vec{y}$ and contains for each conjunction
$l_1\land\dots\land l_n$ in the DNF $\varphi$ a rule
$\pr{sat} \lrule l_1\land \dots \land l_n$
where \pr{sat} is a new atom that is true if $\varphi$ is
satisfied. Moreover, let $\PP_t$ have the rule $\pr{fail} \lrule \lnot
\pr{fail}\land \pr{sat}$.  This rule enforces that \pr{sat} must be
false in models of $\PP_t$. As such $\PP_t$ corresponds to the
sentence $\exists \vec{y}: \lnot \varphi$. Since
$\lnot\exists\vec{y}:\lnot\varphi$ $\equiv$ $\forall\vec{y}:\varphi$,
we thus find that $\exists\vec{x}\forall\vec{y}:\varphi$ is valid iff
$(\PP_g,\PP_t)$ has a stable-unstable model.
\end{example}

It follows from Theorem \ref{thm:complex} that the theoretical
expressiveness of combined programs equals that of \EAQBF{}s.  There
are, however, several reasons why one would prefer combined programs.
Firstly, logic programs are equipped with rich, high-level,
first-order modeling languages.  Secondly, logic programs allow for
natural encodings of \emph{inductive definitions}.  These reasons are
comparable to the advantages of logic programs on the first level of
the hierarchy in contrast with pure SAT. For instance, the former can
naturally express reachability in digraphs, while the latter requires
a non-trivial encoding, which is non-linear in the size of the input
graph.
The advantage of combined programs over \EAQBF{}s is analogous when
solving problems on the second level.  The expressive power of
inductive definitions and the high-level modeling language are
available both in $\PP_g$ and in $\PP_t$. We exploit this when
presenting examples in the next section.

%-----------------------------------------------------------------------------

\section{Applications}
\label{sec:apps}

The goal of this section is to present some applications of
stable-unstable programming. We will focus on \emph{modelling}
aspects, i.e., how certain application problems can be represented.
The programs to be presented are non-ground (and may also use some
constructs present in ASP-Core-2, such as arithmetic) while the
stable-unstable semantics was formulated for ground programs
only. However, in practice, input programs are first grounded and thus
covered by the propositional semantics. Hence, the user has all
high-level primitives of ASP at his/her disposal.

\subsection{Winning Strategies for Parity Games}

\emph{Parity games}, to be detailed below, have been studied
intensively in computer aided verification since they correspond to
model checking problems in the $\mu$-calculus. We show how to
represent parity game instances as combined programs.
A parity game consists of a finite graph $G=(V; A, v_0, V_\exists,
V_\forall, \Omega)$, where $V$ is a set of nodes, $A$ a set of arcs,
$v_0 \in V$ an initial node, $V_\exists$ and $V_\forall$ partition $V$
into two subsets, respectively owned by an existential and a universal
player, and $\Omega : V \to \mathbb{N}$ assigns a priority to each
node.  All nodes are assumed to have at least one outgoing arc.  A
\emph{play} in a parity game is an infinite path in $G$ starting from
$v_0$. We denote such a play by a function $\pi:\mathbb{N} \to V$.  A
play $\pi$ is generated by setting $\pi(0) = v_0$ and, at each step
$i$, asking the player who owns node $\pi(i)$ to choose a following
node $\pi(i+1)$ such that $(\pi(i),\pi(i+1))\in A$. The existential
player wins if $\min\{\Omega(v) \mid v \mbox{ appears infinitely often
  in }\pi\}$ is an even number. Otherwise, the universal player wins.
A \emph{strategy} $\voc_x$ for a player $x \in \{\exists, \forall\}$
is a function that takes a finite path $(v_0, v_1, \cdots, v_n)$ in
$G$ with $v_n \in V_x$ and returns a node $v_{n+1}$ such that $(v_n,
v_{n+1}) \in A$. A play $\pi$ conforms to $\voc_x$ if, whenever
$\pi(n) \in V_x$, it holds that $\voc_x(\pi(0), \pi(1), \cdots,
\pi(n)) = \pi(n+1)$. A strategy $\voc_x$ is a \emph{winning
  strategy} for $x$ if $x$ wins all plays that conform to
$\voc_x$. A strategy $\voc_x$ is called \emph{positional} if
$\voc(v_0, v_1, \cdots, v_n)$ only depends on $v_n$. Two important
properties of parity games are that
(i)
exactly one player has a winning strategy and
(ii)
a player has a winning strategy if and only if it has
a positional winning strategy \cite{EJ91:focs}.

Using the above properties, we provide an intuitive axiomatization
$(\PP_g,\PP_t)$ to capture winning strategies of the existential
player in a given parity game.  The generator program $\PP_g$ is
simple: it guesses a (positional) strategy (called $\pr{eStrategy}$)
for player $\exists$.  The test program is more involved. It guesses a
positional strategy (called $\pr{uStrategy}$) for player $\forall$ and
accepts $\pr{uStrategy}$ if it wins against $\pr{eStrategy}$. To
perform the acceptance test, we define the set $\infpred$ of nodes
that appear infinitely often on the unique play that conforms to both
strategies. We reject $\pr{uStrategy}$ if the minimum priority of
nodes in $\infpred$ is an even number.  Hence, $(\PP_g, \PP_t)$ has a
stable-unstable model if $\PP_g$ can find a positional strategy
$\voc$ for the existential player such that $\PP_t$ cannot find any
positional strategy to defeat $\voc$. The entire programs can be
found below.

\begin{small}
\begin{align*}
\PP_g &= \left\lbrace
\begin{array}{l}
1\leq\#\{\pr{eStrategy}(X,Y) : \pr{arc}(X,Y)\}\leq 1
  \leftarrow \pr{existNode}(X).
\end{array}
\right\rbrace\\
\PP_t &= \left\lbrace
\begin{array}{l}
1\leq \#\{\pr{uStrategy}(X,Y) : \pr{arc}(X,Y)\} \leq 1
  \leftarrow \pr{univNode}(X).\\
\pr{next}(X,Y) \leftarrow \pr{eStrategy}(X,Y).\\
\pr{next}(X,Y) \leftarrow \pr{uStrategy}(X,Y).\\
\pr{r}(v_0). \qquad \pr{r}(Y) \leftarrow \pr{r}(X)\land  \pr{next}(X,Y).\\
\infpred(v_0) \leftarrow \pr{next}(X,v_0)\land  \pr{r}(X).\\
\infpred(X) \leftarrow \pr{next}(Y,X)\land  \pr{next}(Z,X)\land
                       \pr{r}(Y)\land  \pr{r}(Z)\land  Y \neq Z.\\
\infpred(Y) \leftarrow \infpred(X)\land  \pr{next}(X,Y).\\
\pr{infNum}(N) \leftarrow \pr{omega}(X,N)\land \infpred(X).\\
\pr{num}(N) \leftarrow \pr{omega}(X,N).\\
\pr{minNum}(N) \leftarrow \pr{num}(N)\land  N \leq M : \pr{num}(M).\\
\pr{nextNum}(N,M)
\leftarrow \pr{num}(N)\land\pr{num}(M)\land M \leq P : \pr{num}(P) : N < P.\\
\pr{nonMin}(M) \leftarrow \pr{infNum}(N) \land \pr{nextNum}(N,M).\\
\pr{nonMin}(M) \leftarrow \pr{nonMin}(N) \land \pr{nextNum}(N,M).\\
\pr{min}(N) \leftarrow \pr{infNum}(N) \land \lnot \pr{nonMin}(N).\\
\leftarrow \pr{min}(N)\land  N \equiv 0~(\text{mod} 2).
\end{array}
\right\rbrace
\end{align*}
\end{small}

Deciding if a parity game has a winning strategy for the existential
player has been encoded in difference logic and in SAT \cite{HKLN12:jcss}.
We see two reasons why our encoding as a combined
program can still be of interest.
First, it is an intuitive encoding that corresponds directly to the
problem definition.
Second, to the best of our knowledge, it is the first encoding whose
size is linear in the size of the graph, i.e., $\orderof{|V|+|A|}$.
The existing difference logic encoding has size $\orderof{|V|^2+|A|}$
and the existing SAT encoding (which is developed on top of the
difference logic encoding) has size $\orderof{|V|^2\times\log|V|+|A|}$
\cite{HKLN12:jcss}.

\subsection{Conformant Planning}

\emph{(Classical) planning} is the task of generating a plan (i.e., a
sequence of actions) that realizes a certain goal given a complete
description of the world.  \emph{Conformant planning} is the task of
generating a plan that reaches a given goal given a partial
description of the world (certain facts about the initial state and/or
actions' effects are unknown).
In this section, we focus on \emph{deterministic} conformant planning
problems: problems where the state of the world at any time is
completely determined by the initial state and the actions taken.  It
is well-known that deciding if a conformant plan exists is a
$\Sigma^P_2$-complete decision problem.

To encode conformant planning problems in our formalism, we assume a
vocabulary $\voc = \voc_a \cup \voc_w \cup \voc_i$ is given.  Here,
$\voc_a$, $\voc_w$ and $\voc_i$ represent a sequence of actions, the
state of the world over time, and the initial state of the world,
respectively. We also assume that $\voc_w$ contains an atom
$\pr{goal}$ with intended interpretation that the goal of the planning
problem is reached at some time.
Furthermore, we assume that $\voc_i$ is partitioned in $\voc_{unc} $
and $ \voc_{c}$, where $\voc_{unc}$ are the atoms subject to
uncertainty (to which our plan should be conformant). Let $\PP_{ca}$
be a logic program containing a rule $ 0\leq \#\{\pr{a}\}\leq 1$ for
each $\pr{a}\in \voc_a$. Intuitively, the program $\PP_{ca}$
\emph{guesses} a sequence of actions. Similarly, let us introduce a
program $\PP_{unc}$ containing a rule $ 0\leq \#\{\pr{u}\}\leq 1$ for
each $\pr{u}\in \voc_{unc}$.
Furthermore, we assume the availability of a program $\PP_{w}$ that
defines the atoms in $\voc_w$ (including $\pr{goal}$)
deterministically in terms of $\voc_a$ and $\voc_i$.  Also, let
$\PP_{pa}$ be a program that contains a rule $\pr{fail} \lrule
\pr{a}\land \lnot \pr{p}$ for each $\pr{a}\in \voc_a$, $\pr{p}\in
\voc_w$ such that $\pr{p}$ is a precondition of $\pr{a}$.
With these building blocks, we can easily encode conformant planning
as a combined program
\[
\left(\PP_{ca}, \PP_w\cup \PP_{pa}  \cup \PP_{unc}\cup
\{\lrule \pr{goal}\land \lnot \pr{fail}\}\right).
\]
This program is parameterized by $\voc_{c}$.  To see that it encodes
the conformant planning problem, we notice that stable-unstable models
of this program are stable models of $\PP_{ca}$, i.e., sequences of
actions. Furthermore, models of the testing program are
interpretations of the atoms in $\voc_{unc}$ such that in this world,
either one of the preconditions on the actions is not satisfied or the
goal is not reached. I.e., models of the testing program amount to
showing that the sequence of actions is \emph{not} a conformant
plan. The stable-unstable semantics dictates that there can be no such
counterexample.

In the above, we described $\PP_{w}$ and $\PP_{pa}$ only informally
since these components have already been worked out in the
literature. More precisely, many classical planning encodings use
exactly those components, combining them to a program of the form
\[\PP= \PP_{ca}\cup \PP_{w} \cup\PP_{pa}\cup
  \{\lrule \lnot \pr{goal}. ~~\lrule \pr{fail}.\}.
\]
These components (or very similar) are used for instance by
\citet{Lifschitz99:iclp}, \citet{LRS01:puui}, and by \citet{BJBDVD14:tplp}.
This illustrates that our encoding of conformant planning stays very
close to the existing encodings of classical planning problems in ASP.
On the other hand, native conformant planning encodings in ASP are
often based on saturation \cite{LRS01:puui}. After applying saturation, it
is very hard to spot the original components.

\subsection{%
Points of No Return: 
A Generic Problem Combining Logic and Graphs}

We now present a generic problem that connects graphs with logic.  Let
$G=(V,A,s)$ be a directed multi-graph: $V$ is a set of nodes, $s \in
V$ is an initial node and $A$ is a set of arcs labeled with Boolean
formulas. We use $a:\arc{u}{\phi}{v}$ to denote that $a$ is an arc
from $u $ to $v $ labeled with $\phi$.  There may be multiple arcs
between $u$ and $v$ with different labels. We call a node $v \in V$ a
\emph{point of no return} if
(i)
$G$ contains a path 
$\arc{s=v_0}{\phi_1}{v_1} \arc{}{\phi_2}{\ldots}\arc{}{\phi_n}{v_n=v}$
such that $\phi_1\land \ldots\land \phi_n$ is satisfiable and
(ii)
the preceding path in $G$ cannot be extended with a path
$\arc{v=v_n}{\phi_{n+1}}{v_{n+1}}
 \arc{}{\phi_{n+2}}{\ldots}
 \arc{}{\phi_{n+m}}{v_{n+m}=s}$
such that $\phi_1\land \dots \land \phi_n \land \phi_{n+1}\land \ldots
\land\phi_{n+m}$ is satisfiable.
Thus, points of no return are nodes $v$ that can be reached from $s$
in a way that makes $s$ unreachable from $v$ (i.e., reaching $s$ back
from $v$ would violate a constraint of the path from $s$ to $v$).

\begin{proposition}
Given a finite labeled graph $G=(V,A,s)$ as above and a node $v \in
V$, it is a $\Sigma^P_2$-complete problem to decide if $v$ is a
point of no return.
\end{proposition}

\begin{proof}
Membership in $\Sigma^P_2$ is obvious. We present a reduction from \EAQBF to
support hardness. Consider an \EAQBF formula $\exists x_1\cdots\exists x_n
\forall y_1\cdots\forall y_m\phi$. This formula is equivalent to
\[\exists x_1 \cdots\exists x_n\neg\exists y_1\cdots\exists y_m\neg\phi.\]
Now, construct a graph $G$ with nodes
$v_0,v_1,$ \ldots, $v_n,v_{n+1},$ \ldots, $v_{n+m+1}$
and following labeled arcs:
$$\begin{array}{lcr}
\arc{v_{i-1}}{x_i}{v_{i}} \mbox{ and } \arc{v_{i-1}}{\neg x_i}{v_{i}}
  & \quad & \mbox{(for $1\leq i\leq n$)},\\
\arc{v_{n+j}}{y_j}{v_{n+j+1}} \mbox{ and } \arc{v_{n+j}}{\neg y_j}{v_{n+j+1}}
  & & \mbox{(for $1\leq j\leq m$)},\\
\arc{v_n}{\neg\phi}{v_{n+1}} \mbox{ and }\arc{v_{n+m+1}}{\top}{v_0}.
\end{array}$$
Observe that, setting $s=v_0$ and $v=v_n$, we have that $v$ is a point
of no return if and only if
$\exists x_1 \cdots\exists x_n \forall y_1\cdots\forall y_m\phi$
is valid.
\end{proof}

To model the problem of checking whether a node is point of no return
as a combined program, we assume that each arc is labeled by a
\emph{literal} and that there is at most one arc between every two
nodes. Our programs easily generalize to the general case. To allow
for multiple arcs between two nodes, it suffices to introduce explicit
identifiers for arcs. To allow more complex labeling formulas, we can
introduce Tseitin predicates for subformulas and use standard
meta-interpreter approaches to model the truth of such a formula; see
for instance \cite[Section 3]{GKS11:tplp}.

We use unary predicates $\pr{init}$ and $\pr{ponr}$ to respectively
interpret the initial node $s$ and the point of no return
$v$. Herbrand functions $\pr{pos}$ and $\pr{neg}$ map atoms
(represented as constants) to literals. The predicate
$\pr{arc}(X,Y,L)$ holds if there is an arc between nodes $X$ and $Y$
labeled with literal $L$.  In $\PP_g$ (and $\PP_t$), we use predicates
$\pr{pick}_g$ (and $\pr{pick}_t$) such that $\pr{pick}_g(X,Y)$
($\pr{pick}_t(X,Y)$) holds if the arc from $X$ to $Y$ is chosen in the
path $v_0\to\dots\to v_n$ (the path $v_n\to\dots\to v_{n+m}$
respectively).
The programs contain constraints ensuring that the selected edges
indeed form paths from $s$ to $v$ (respectively from $v$ to $s$),
using an additional predicate $\pr{r}_g$ ($\pr{r}_t$) and that the
formulas associated to the respective paths are satisfiable. Thus,
$\PP_g$ encodes that there exists a path from $s$ to $v$ and $\PP_t$
encodes that this path can be extended to a cycle back to $s$. As
such, the combined program indeed models that $v$ is a point of no
return. The entire combined program can be found below.

{\small
\[\begin{array}{@{}c@{~}c@{}}
\PP_g & \PP_t\\
=&=\\
\left\lbrace
\begin{array}{l@{~}}
0\leq\# \{\pr{pick}_g(X,Y)\}\leq 1 \lrule \pr{arc}(X,Y,L).\\
\lrule \pr{pick}_g(X,Y)\land \pr{pick}_g(X',Y') \\
\quad \land\, \pr{arc}(X,Y,\pr{pos}(A)) \land \pr{arc}(X',Y',\pr{neg}(A)).\\
\pr{r}_g(X) \leftarrow \pr{init}(X).\\
\pr{r}_g(Y) \leftarrow \pr{r}_g(X)\land \pr{pick}_g(X,Y).\\
\lrule \lnot \pr{r}_g(X)\land \pr{pick}_g(X,Y).\\
\lrule \pr{ponr}(X) \land \lnot \pr{r}_g(X).\\
\lrule \pr{ponr}(X) \land \pr{pick}_g(X, Y).\\
\lrule \pr{pick}_g(X,Y)\land \pr{pick}_g(X,Z)\land Y\neq Z.\\
\lrule \pr{pick}_g(X,Y)\land \pr{pick}_g(Z,Y)\land X\neq Z.\\
\end{array}
\right\rbrace & \left\lbrace
\begin{array}{l@{~}}
0\leq\# \{\pr{pick}_t(X,Y)\}\leq 1 \lrule \pr{arc}(X,Y,L).\\
\pr{pick}(X,Y) \lrule \pr{pick}_t(X,Y).\\
\pr{pick}(X,Y)\lrule \pr{pick}_g(X,Y).\\
\lrule \pr{pick}(X,Y)\land \pr{pick}(X',Y') \land \\
\quad \pr{arc}(X,Y,\pr{pos}(A)) \land \pr{arc}(X',Y',\pr{neg}(A)).\\
\pr{r}_t(X) \leftarrow \pr{ponr}(X).\\
\pr{r}_t(Y) \leftarrow \pr{r}_t(X)\land \pr{pick}_t(X,Y).\\
\lrule \lnot \pr{r}_t(X)\land \pr{pick}_t(X,Y).\\
\lrule \pr{init}(X) \land \lnot \pr{r}_t(X).\\
\lrule \pr{init}(X) \land \pr{pick}_t(X, Y).\\
\lrule \pr{pick}_t(X,Y)\land \pr{pick}_t(X,Z)\land Y\neq Z.\\
\lrule \pr{pick}_t(X,Y)\land \pr{pick}_t(Z,Y)\land X\neq Z.\\
\end{array}
\right\rbrace
\end{array}
\]}

%-----------------------------------------------------------------------------

\section{Implementation}
\label{sec:impl}

Next, we present a prototype implementation of a solver for the
stable-unstable semantics.

\subsection{Preliminaries: \system{sat-to-sat}}

We assume familiarity with the basics of second-order logic (SO).  Our
implementation is based on a recently introduced solver, called \system{sat-to-sat}
\cite{JTT16:aaai}. The \system{sat-to-sat} architecture combines multiple SAT solvers
to tackle problems from any level of the PH, essentially acting like a
QBF solver \cite{BJT16:bnp}. We do not give details on the inner
workings of \system{sat-to-sat}, but rather refer the reader to the original papers
for details.
What matters for the current paper is that \citet{BJT16:kr}
presented a high-level (second-order) interface to \system{sat-to-sat}.  The idea is
that in order to obtain a solver for a new paradigm, it suffices to
give a second-order theory that \emph{describes} the semantics of the
formalism declaratively.  Bogaerts et al.\ showed, e.g., how to obtain
a solver for (disjunctive) logic programming using this idea.

Following \citet{BJT16:kr}, we describe a logic program by means
of predicates $\pr{r}$, $\pr{a}$, $\pr{p}$, $\pr{h}$, $\pr{pb}$ and
$\pr{nb}$ with intended interpretation that $\pr{r}(R)$ holds for all
rules $R$, $\pr{a}(A)$ holds for all atoms $A$, $\pr{p}(A)$ holds for
all parameters, $\pr{h}(R,H)$ means that $H$ is an atom in the head of
rule $R$, $\pr{pb}(R,A)$ that $A$ is a positive literal in the body of
$R$ and $\pr{nb}(R,B)$ that $B$ is the atom of a negative literal in
the body of $R$.  With this vocabulary, augmented with a predicate
$\pr{i}$ with intended meaning that $\pr{i}(A)$ holds for all atoms
$A$ true in some interpretation, we describe the parameterized stable
semantics for disjunctive logic programs with the theory
$\theory_{SM}$:

\begin{small}
\begin{equation*}
\begin{array}{@{}l@{}}
\left\{\begin{array}{l} 
\forall A: \pr{i}(A)\limplies \pr{a}(A).\\
\forall R: \pr{r}(R)\limplies
  \big((\forall A: \pr{pb}(R,A)\limplies \pr{i}(A))\land
       (\forall B: \pr{nb}(R,B)\limplies \lnot \pr{i}(B)) \limplies \\
\qquad\qquad\quad~~\,
       \exists H: \pr{h}(R,H)\land \pr{i}(H)\big). \\
\lnot \exists \pr{i}':\\
\quad
(\forall A: \pr{i}'(A)\limplies \pr{i}(A))\land
(\exists A: \pr{i}(A)\land \lnot \pr{i}'(A))\land
(\forall A: \pr{p}(A) \limplies (\pr{i}'(A)\lequiv \pr{i}(A)))\,\land \\
\quad
\forall R: \pr{r}(R)\limplies
  \big((\forall A: \pr{pb}(R,A)\limplies \pr{i}'(A))\,\land \\
\qquad\qquad\qquad~~\,
       (\forall B: \pr{nb}(R,B)\limplies \lnot \pr{i}(B))
       \limplies \exists H: \pr{h}(R,H)\land \pr{i}'(H)\big). \\
\end{array}
\right\}
\end{array}
\end{equation*}
\end{small}

The first part of this theory expresses that $\pr{i}$ is interpreted
as a model of $\PP$: the constraint $\pr{i}(A)\limplies \pr{a}(A)$
expresses that the interpretation is a subset of the vocabulary and
the second constraint expresses that whenever the body of a rule is
satisfied in $\pr{i}$, so is at least one of its head atoms.  The
constraint $\lnot \exists \pr{i}'\dots$ expresses that $i$ is
$\subseteq$-minimal: there cannot be an interpretation
$\pr{i}'\subsetneq \pr{i}$ that agrees with $\pr{i}$ on the parameters
and that is a model of the reduct of $\PP$ with respect to
$\pr{i}$. In other words, whenever $\pr{i}'$ satisfies all positive
literals in the body of a rule $R$ and $\pr{i}$ satisfies all negative
literals in the body of $R$, $\pr{i}'$ must also satisfy some atom in
the head of $R$.
\begin{theorem}[Theorem 4.1 of \cite{BJT16:kr}]
\label{thm:sm}\label{thm:stable}
Let $\PP$ be a (disjunctive) logic program and $I$ an interpretation
that interprets
$\{\pr{a},\pr{r},\pr{p},\pr{pb},\pr{nb},$ $\pr{h}\}$
according to $\PP$. Then, $I\models \theory_{SM}$ if and only if
$\pr{i}^I$ is a parameterized stable model of $\PP$.
\end{theorem}

From Theorem \ref{thm:sm}, it follows that feeding $\theory_{SM}$ to
\system{sat-to-sat} results in a solver for disjunctive logic programs.  The same
theory also works for normal logic programs.

\subsection{An Implementation on Top of \system{sat-to-sat}}

In order to obtain a solver for our new paradigm in the spirit of
\citet{BJT16:kr}, we need to provide a second order specification
of our semantics.  A first observation is that we can reuse the theory
$\theory_{SM}$ from the previous section, both to enforce that $I$ is
a stable model of $\PP_g$ and that there exists no stable model of
$\PP_t$ that coincides with $I$ on the shared vocabulary.
When translating the definition of stable-unstable models to
second-order logic, we obtain the following theory
\begin{equation*}%\label{eq:stable-unstable}
 \begin{array}{l}
\theory_{SU} = 
\left\{\begin{array}{l} 
\theory_{SM}[\pr{r}/\pr{r}_g,\pr{a}/\pr{a}_g,\pr{p}/\pr{p}_g,
            \pr{h}/\pr{h}_g,\pr{pb}/\pr{pb}_g,\pr{nb}/\pr{nb}_g] . \\ 
\lnot \exists \pr{i}_t:
  \theory_{SM}[\pr{r}/\pr{r}_t,\pr{a}/\pr{a}_t,\pr{h}/\pr{h}_t,
              \pr{pb}/\pr{pb}_t,\pr{nb}/\pr{nb}_t,\pr{i}/\pr{i}_t,
              \pr{p}/\pr{p}_t]\\ 
\qquad\qquad \land \,  
(\forall A: \pr{a}_g(A)\land \pr{a}_t(A)\limplies
            (\pr{i}(A)\lequiv \pr{i}_t(A))).
\end{array}\right\},
\end{array}
\end{equation*}
where $\theory_{SM}[\pr{r}/\pr{r}_g]$ abbreviates a second-order
theory obtained from $\theory_{SM}$ by replacing all free occurrences
of $\pr{r}$ by $\pr{r}_g$.
\begin{theorem}
Let $(\PP_g,\PP_t)$ be a combined logic program and $I$ an
interpretation that interprets $\{\pr{a}_g,\pr{r}_g,\pr{p}_g,$
$\pr{pb}_g,\pr{nb}_g,\pr{h}_g\}$ according to $\PP_g$ and
$\{\pr{a}_t, \pr{r}_t, \pr{p}_t, \pr{pb}_t, \pr{nb}_t, \pr{h}_t\}$
according to $\PP_t$. Then, $I\models \theory_{SU}$ if and only if
$\pr{i}^I$ is a stable-unstable model of $(\PP_g,\PP_t)$.
\end{theorem}
\begin{proof}
Theorem \ref{thm:stable} ensures that the first sentence of this
theory is equivalent with the condition of $\pr{i}^I$ being a stable
model of $\PP_g$.  Also, the second sentence states that one cannot
have an interpretation $\pr{i}_t$ that coincides with $\pr{i}^I$ on
shared atoms (those that are in both $\pr{a}_g$ and $\pr{a}_t$) and
is a stable model of $\PP_t$. This is exactly the definition of the
stable-unstable semantics.
\end{proof}

Providing an ASCII representation of $\theory_{SU}$ to the
second-order interface of \system{sat-to-sat} immediately results in a solver that
generates stable-unstable models of a combined logic program.  Our
implementation, which is available online%
\footnote{%
\url{http://research.ics.aalto.fi/software/sat/sat-to-sat/so2grounder.shtml}.},
consists only of the second-order theory above and some marshaling (to
support ASP-Core-2 format and to exploit the symbol table to identify
which atoms from different programs are actually the same).  The
overall workflow of our tool is as follows. We take, as input, three
logic programs: $\PP_g$ (a non-ground generate program), $\PP_t$ (a
non-ground test program) and $\PP_i$ (an instance). We then use
\system{Gringo} \cite{GST07:lpnmr} to ground $\PP_g \cup \PP_i$ and
$\PP_t \cup \PP_i$. Next, we interpret $\pr{a}_x$, $\pr{r}_x$,
$\pr{p}_x$, $\pr{pb}_x$, $\pr{nb}_x$ and $\pr{h}_x$ (for $x \in
\{g,t\}$) according to the reified representation of the two resulting
ground programs. Such an interpretation is fed to \system{sat-to-sat} along with the
ASCII representation of $\theory_{SU}$; \system{sat-to-sat} uses these to compute
stable-unstable models of the original combined program $(\PP_g \cup
\PP_i, \PP_t \cup \PP_i)$.

The implementation described above is proof-of-concept by nature and
we plan to implement this technique natively on top of the
\system{clasp} solver
\cite{DGGKKOS08:kr,GKKRS15:lpnmr}.
In spite of its prototypical nature, the current implementation is based
on a state-of-the-art architecture shared by many QBF solvers and thus
expected to perform reasonably well. This is especially the case when
we go beyond the complexity class $\Sigma^P_2$ in the next section.

%-----------------------------------------------------------------------------

\section{%
Beyond $\Sigma^P_2$ with Normal Logic Programs}
\label{sec:beyond}

In this section, we show how the ideas of this paper
generalize to capture the entire PH. To this end, the
definition of a combined logic program is turned into a recursive
definition of $k$-combined programs where the parameter $k\geq 1$
reflects the \emph{depth} of the combination.

\begin{definition}[$k$-combined program]
\begin{compactenum}
\item
For $k=1$, a \emph{$1$-combined program} is defined as a normal
program $\PP$ over a vocabulary $\voc$, parameterized by a
vocabulary $\tau\subseteq\voc$.

\item
For $k>1$, a \emph{$k$-combined} program is a pair
$(\PP,\CC)$ where
$\PP$ is a normal program over a vocabulary $\voc$,
parameterized by a vocabulary $\tau\subseteq\voc$ and
$\CC$ is a $(k-1)$-combined program
over a vocabulary $\voc'$, parameterized by
 $\voc\cap\voc'$.
\end{compactenum}
\end{definition}
Note that \emph{combined programs} (Definition
\ref{def:combined-program}) directly correspond to $k$-combined
programs with $k=2$.  Similarly, the semantics of $k$-combined
programs also directly generalizes Definition \ref{def:semantics}:
\begin{definition}[Stable-unstable models for $k$-combined programs]
A stable model $I$ of \PP is also called a \emph{stable-unstable}
model of a $1$-combined program $\PP$.  Let $(\PP,\CC)$ be a
$k$-combined program with $k>1$ over a vocabulary $\voc$,
parameterized by $\tau\subseteq\voc$, where $\CC$ has vocabulary
$\voc'$.
A $\voc$-interpretation $I$ is a \emph{stable-unstable model} of
$(\PP,\CC)$, if
\begin{compactenum}
\item
$I$ is a parameterized stable model of $\PP$  and

\item
there is no stable-unstable model $J$ of $\CC$ such that $I\cap
{\voc'}=J\cap {\voc}$.
\end{compactenum} 
\end{definition}
\begin{example}[Example \ref{ex:small} continued]
Consider program $\PP_3=\{e\lrule e.\ d\lrule e.\}$ over vocabulary
$\voc_3=\{d,e\}$.  Program $\PP_3$ has one stable model, namely
$\emptyset$. This model is also a stable-unstable model of the
$3$-combined program $(\PP_3,(\PP_2,\PP_1))$ since it does not
coincide with a stable-unstable model of $(\PP_2,\PP_1)$ on
$\voc_3\cap \voc_2=\{d\}$.
\end{example}
The complexity of deciding whether a $k$-combined
program $(\PP,\CC)$ has a stable-unstable model depends on the depth
$k$ of the combination.

\begin{theorem}\label{thm:complex:general}
It is $\Sigma^P_k$-complete to decide if a finite $k$-combined
program has a stable-unstable model.
\end{theorem}

\begin{proof}[Proof sketch.]
The case $k=1$ follows from the results of
\citet{MT99} and
Theorem \ref{thm:complex} corresponds to $k=2$.
Using either one as the base case, it can be proven inductively that
the decision problem in question is NP-complete assuming the
availability of an oracle from the class $\Sigma^P_{k-1}$,
effectively a $(k-1)$-combined program in our constructions.
Thus, steps in recursion depth match with the levels of the PH
(in analogy to the number of quantifier alternations in QBFs).
\end{proof}

%-----------------------------------------------------------------------------

\section{Conclusion}
\label{sec:concl}

In this paper, we propose \emph{combined logic programs} subject to
the \emph{stable-unstable semantics} as an alternative paradigm to
disjunctive logic programs for programming on the second
level of the polynomial hierarchy. We deploy \emph{normal} logic
programs as the base syntax for combined programs, but other equally
complex classes can be exploited analogously. Our methodology
surpasses the need for saturation and meta-interpretation techniques
that have previously been used to encode oracles within disjunctive
logic programs.
The use of the new paradigm is illustrated in terms of application
problems and we also present a proof-of-concept implementation on top
of the solver \system{sat-to-sat}. Moreover, we show how combined programs provide a
gateway to programming on any level $k$ of the polynomial hierarchy
with normal logic programs using the idea of recursive combination
to depth $k$. In this sense, our formalism can be seen as a hybrid
between QBFs and logic programs, combining desirable features from
both.

%-----------------------------------------------------------------------------

\label{lastpage}
\end{document}